\documentclass[11pt]{article}
\usepackage[utf8]{inputenc}
\usepackage{algorithm}
\usepackage{algorithmic}
\usepackage{natbib}
\usepackage{amsthm}
\usepackage{amssymb}
\usepackage{bbm}
\usepackage{amsmath}
\usepackage{graphicx}
\usepackage{xcolor}
\usepackage{hhline}
\usepackage[breaklinks=true]{hyperref}
\usepackage{breakcites}

\hypersetup{
    colorlinks,
    linkcolor={blue!50!black},
    citecolor={blue!50!black},
}
\colorlet{linkequation}{blue}

\newtheorem{thm}{Theorem}[section]
\newtheorem{cor}[thm]{Corollary}
\newtheorem{lem}[thm]{Lemma}

\newtheorem{defn}[thm]{Definition}

\advance \oddsidemargin by -0.8in
\newcommand{\myparskip}{3pt}
\parskip \myparskip

\title{\textbf{A Note on the Representation Power of GHHs}}
\author{
  Zhou Lu\footnote{This work is done during LZ's visit to SQZ institution.}\\
  Princeton University\\
  \texttt{zhoul@princeton.edu}
}
\date{January 2021}
\begin{document}

\maketitle

\begin{abstract}
    In this note we prove a sharp lower bound on the necessary number of nestings of nested absolute-value functions of generalized hinging hyperplanes (GHH) to represent arbitrary CPWL functions. Previous upper bound states that $n+1$ nestings is sufficient for GHH to achieve universal representation power, but the corresponding lower bound was unknown. We prove that $n$ nestings is necessary for universal representation power, which provides an almost tight lower bound. We also show that one-hidden-layer neural networks don't have universal approximation power over the whole domain. The analysis is based on a key lemma showing that any finite sum of periodic functions is either non-integrable or the zero function, which might be of independent interest.
\end{abstract}

\section{Introduction}
We consider the complexity of representing continuous piecewise linear functions using the generalized hinging hyperplane model \cite{wang2005generalization}. We begin with a short review on these two notions.

\subsection{Continuous Piecewise Linear (CPWL) Functions}
Continuous piecewise linear (CPWL) functions play an important role in non-linear function approximation, such as nonlinear circuit or neural networks. We introduce the definition of CPWL functions borrowed from \cite{chua1988canonical}.

\begin{defn}[CPWL function]\label{def:cpwl}
A function $f(x): R^n \to R$ is said to be a CPWL function iff it satisfies:\\

1):The domain space $R^n$ is divided into a finite number of polyhedral regions by a finite number of disjunct boundaries. Each boundary is a subset of a hyperplane and takes non-zero measure (standard lebesgue measure) on the hyperplane (as $R^{n-1}$).

2):The restriction of $f(x)$ on each polyhedral region is an affine function.

3):$f(x)$ is continuous on $R^n$.
\end{defn}

\subsection{Generalized Hinging Hyperplanes (GHH)}
The model of hinging hyperplanes (HH) is a sum of hinges like
\begin{equation}
\pm \max \{w_1^{\top} x+b_1, w_2^{\top} x+b_2 \}
\end{equation}
where $w_1, w_2\in R^n$ and $b_1, b_2\in R$ are parameters. The HH model (in fact equivalent to a one hidden-layer ReLU network) can approximate any continuous function over a compact domain to arbitrary precision as the number of hinges go infinity \cite{breiman1993hinging}.

However, this model can't exactly represent all CPWL function as pointed out in \cite{he2018relu}, which brings doubt on its approximation efficiency. To overcome this problem, \cite{wang2005generalization} first proposed a generalization of HH model, called GHH which allows more than 2 affine functions within the nested maximum operator:

\begin{defn}[$n$-order hinge]
A $n$-order hinge is a function of the following form:
\begin{equation}
\pm \max \{w_1^{\top} x+b_1, w_2^{\top} x+b_2, \cdots, w_{n+1}^{\top} x+b_{n+1}\}
\end{equation}
where $w_i\in R^n$ and $b_i\in R$ are parameters.
\end{defn}

A linear combination of a finite number of $n$-order hinges is called a $n$-order hinging hyperplane ($n$-HH) model. Such model has universal representation power over all CPWL functions, as formalized in the theorem below:

\begin{thm}[Theorem 1 in \cite{wang2005generalization}]\label{thm:upper}
For any positive integer $n$ and CPWL function $f(x): R^n \to R$, there exists a $n$-HH which exactly represents $f(x)$.
\end{thm}

The question is whether we can give a sharp lower bound on the necessary number of affine functions within the nested maximum operator. \cite{wang2005generalization} conjected that $(n-1)$-HH can't represent all CPWL functions, but this open problem is left unanswered for more than a decade. In the following section we will prove our main result that $(n-2)$-HH can't represent all CPWL functions, yielding an almost tight lower bound.

\section{Main Result}
Observe that any $(n-2)$-order hinge depends on only $n-1$ affine transforms of $x$, thus there always exists a direction in which the value of the $(n-2)$-order hinge remains the same. We make such observation precise by introducing the definition of low-dimensional and periodic functions.

\begin{defn}[Low-dimensional/periodic function]
A function $f(x): R^n \to R$ is said to be low-dimensional, if there exists a vector $v\ne 0$, such that for any $x\in R^n$ and $c\in R$, we have that $f(x)=f(x+cv)$. If we have only $f(x)=f(x+v)$ then $f$ is said to be periodic (a weaker notion). $v$ is called an invariant vector of $f$.
\end{defn}

Any $(n-2)$-order hinge is a low-dimensional function on $R^n$, so our problem is reduced to proving the class of finite sum of low-dimensional functions has limited representation power. The following key lemma actually proves (a stronger result) that finite sum of periodic functions can't represent any non-trivial integrable functions.

\begin{lem}\label{lem:key}
Any finite sum of periodic functions is either non-integrable or the zero function, i.e. given periodic functions $f_i(x)$, $i=1,...,m$, then $f(x)\triangleq \sum_{i=1}^m f_i(x)$ satisfies
\begin{equation}
    \int_{R^n} |f| =\infty \quad or \quad f\equiv 0
\end{equation}
\end{lem}

\begin{proof}
We prove Lemma \ref{lem:key} by induction. Suppose each $f_i$ has an invariant vector $v_i$, base case $m=1$ is trivial since if we denote the orthogonal hyperplane $H_i=\{x|x^{\top} v_i=0\}$, we have
\begin{equation}
    \int_{R^n} |f|=\int_R \int_{H_1} |f|
\end{equation}
thus $\int_{R^n} |f|<\infty$ if and only if $\int_H |f|=0$. Assume $f=\sum_{i=1}^m f_i$ is integrable, then $g(x)\triangleq f(x+v_m)-f(x)$ is also integrable. We make the following decomposition of $g$:
\begin{equation}
    g(x)=\sum_{i=1}^m f_i(x+v_m)-f_i(x)=\sum_{i=1}^{m-1} f_i(x+v_m)-f_i(x)
\end{equation}
where each $f_i(x+v_m)-f_i(x)$ is periodic (with invariant vector $v_i$) as well. By induction we have $g\equiv 0$ and $f$ is also a periodic function (with invariant vector $v_m$). Using the base case on $f$ again concludes our proof.
\end{proof}

Our main result is a direct corollary of Lemma \ref{lem:key}, as stated below:
\begin{thm}\label{thm:main}
For any positive integer $n\ge 2$, there exists a CPWL function $g(x): R^n \to R$, such that no $(n-2)$-HH can exactly represent $g(x)$.
\end{thm}
\begin{proof}
Let $g(x)\triangleq \max\{0, 1-||x||_{\infty}\}$. It's straightfoward to check that $g(x)$ is a CPWL function with at most $2^{n+1}$ affine polyhedral regions, and meanwhile is an integrable function with positive integral. As any $(n-2)$-HH can be written as a finite sum of low-dimensional functions, it can't represent $g(x)$ by Lemma \ref{lem:key}.
\end{proof}

Theorem \ref{thm:main} implies that in order to achieve universal representation power over all CPWL functions, a $(n-1)$-HH model is necessary which provides an almost tight lower bound corresponding to the upper bound in Theorem \ref{thm:upper}.

\section{Implications on Universal Approximation of ANNs}
Traditional universal approximation theorems of artifical neural networks (ANN) \cite{cybenko1989approximation, hornik1989multilayer, barron1994approximation} typically states that an ANN with one hidden layer and unbounded width can approximate any measurable function with arbitrary precision on a compact set. Our result demonstrates that the compact set assumption is indeed necessary for ANNs with traditional activation (composition of an affine transform and a fixed univariate function $\sigma$):
\begin{cor}\label{cor:NN}
Given an integrable function $f$ on $R^n$ ($n\ge2$), for any one-hidden-layer neural network $g$ with traditional activation $\sigma(w^{\top}x+b)$, we have that
\begin{equation}
    \int_{R^n} |f-g|=\infty \quad or \quad \int_{R^n} |f-g|=\int_{R^n}|f|
\end{equation}
\end{cor}
\begin{proof}
Any unit $\sigma(w^{\top}x+b)$ is obviously a low-dimensional function when $n\ge 2$, thus by Lemma \ref{lem:key} we finish our proof.
\end{proof}
Corollary \ref{cor:NN} reveals a fundamental gap of representation power between one-hidden layer neural networks and deeper ones, as Theorem \ref{thm:upper} indicates a neural network with $\lceil log_2(n+1)\rceil$ hidden layers can represent any CPWL function \cite{he2018relu}, showing the benefits of depth in universal approximation \cite{lu2017expressive}.

\section{Conclusion}
In this note we give a sharp lower bound on the necessary number of nestings of nested absolute-value functions of generalized hinging hyperplanes (GHH) to represent arbitrary CPWL functions, which is the first non-trivial lower bound to the best of our knowledge. Our results fully characterizes the representation power (and limit) of the GHH model.

Our result also has implications on ANNs, a much more popular model in machine learning. It shows that one-hidden-layer neural networks with traditional activation can't control the approximation error on the whole domain despite existing universal approximation theorems, a fundamental gap between one-hidden-layer networks and deeper ones. We conject similar depth-separation results should hold for deeper networks and the $\lceil log_2(n+1)\rceil$ bound should be tight in representing CPWL functions. Instead of low-dimensional (periodic), other properties need to be discovered for deeper networks.

\section*{Acknowledgements}
The author would like to thank Fedor Petrov for giving an elegant proof of Lemma \ref{lem:key} on Mathoverflow.

\bibliography{Xbib}
\bibliographystyle{plainnat}

\end{document}